\documentclass{amsart} 
\usepackage{apptools}
\AtAppendix{\counterwithin{theorem}{section}}

\usepackage{nicefrac}
\usepackage{bbm}
\usepackage{amsthm, amssymb, amsmath}
\usepackage{graphicx}
\usepackage{hyperref}
\usepackage[round]{natbib}

\DeclareSymbolFont{matha}{OML}{txmi}{m}{it}
\DeclareMathSymbol{\varv}{\mathord}{matha}{35}

\theoremstyle{definition}

\newtheorem{theorem}{Theorem}
\newtheorem{corollary}{Corollary}

\newtheorem{proposition}{Proposition}
\newtheorem{lemma}{Lemma}
\newtheorem{definition}{Definition}
\newtheorem{example}{Example}

\theoremstyle{remark}

\newtheorem{rmk}{Remark}

\DeclareRobustCommand{\norm}[1]{\left\lVert #1 \right\rVert} 
\DeclareRobustCommand{\bb}[1]{\mathbb{#1}} 
\DeclareRobustCommand{\c}[1]{\mathcal{#1}}

\DeclareRobustCommand{\t}[1]{\text{#1}}

\DeclareRobustCommand{\Set}[2][]{\left\{#1 \, \middle|\, #2 \right\}}
\DeclareRobustCommand{\bk}[2]{\left\langle #1,\,#2\right\rangle}

\DeclareRobustCommand{\f}[2]{\frac{#1}{#2}}
\DeclareRobustCommand{\ep}{\varepsilon}
\DeclareRobustCommand{\eps}{\epsilon}
\DeclareRobustCommand{\d}{\delta}

\DeclareRobustCommand{\a}{\alpha}

\DeclareRobustCommand{\vp}{\varphi}

\DeclareRobustCommand{\mr}[1]{\mathrm{#1}}

\DeclareRobustCommand{\th}{\t{th}}

\newcommand\restr[2]{{
  \left.\kern-\nulldelimiterspace 
  #1 
  \vphantom{\rule{0pt}{9pt}} 
  \right|_{#2} 
  }}

\DeclareMathOperator{\conv}{conv}
\DeclareMathOperator{\vol}{vol}
\DeclareMathOperator{\vspan}{span}
\DeclareMathOperator{\tr}{tr}

\title{Frank-Wolfe Meets Metric Entropy}

\begin{document}
\begin{abstract}%
    The Frank-Wolfe algorithm has seen a resurgence in popularity due to its ability to efficiently solve constrained optimization problems in machine learning and high-dimensional statistics. 
    As such, there is much interest in establishing when the algorithm may possess a ``linear'' $O(\log(1/\eps))$ dimension-free iteration complexity comparable to projected gradient descent.

    In this paper, we provide a general technique for establishing domain specific and easy-to-estimate lower bounds for Frank-Wolfe and its variants using the metric entropy of the domain. Most notably, we show that a dimension-free linear upper bound must fail not only in the worst case, but in the \emph{average case}: for a Gaussian or spherical random polytope in $\bb{R}^d$ with $\mr{poly}(d)$ vertices, Frank-Wolfe requires up to $\tilde\Omega(d)$ iterations to achieve a $O(1/d)$ error bound, with high probability. We also establish this phenomenon for the nuclear norm ball. 

    The link with metric entropy also has interesting positive implications for conditional gradient algorithms in statistics, such as gradient boosting and matching pursuit. In particular, we show that it is possible to extract fast-decaying upper bounds on the excess risk directly from an analysis of the underlying optimization procedure.
\end{abstract}

\author{Suhas Vijaykumar}
\address{MIT Dept. of Economics and Center for Statistics}
\email{suhasv@mit.edu}

\maketitle

\setcounter{tocdepth}{1}
\tableofcontents

\clearpage

\section{Introduction \& Related Work}
Constrained, high-dimensional convex optimization problems are ubiquitous in modern statistics and machine learning, and many promising approaches have emerged to solve them. In particular, for well-behaved objectives, the projected gradient descent algorithm of \citet{beck_fast_2009} has been shown to need only $O(\log(1/\eps))$ iterations to find an $\eps$-aproximate solution \citep{bubeck2015convex}. Unfortunately, it can require a costly ``projection step'' at each iteration, which is particularly expensive in high dimensional problems.

In contrast, conditional gradient algorithms such as Frank-Wolfe only require the solution of a linear program at each iteration, which is a major improvement when the ambient dimension is large \citep{frank_algorithm_1956,jaggi_revisiting_2013}. However, even for well-behaved objectives, the worst case dimension-free iteration complexity scales as $\Omega(1/\eps)$. This has led to a search for algorithmic variants and high-level conditions that produce a ``linear'' $O(\log(1/\eps))$ iteration complexity, where the polyhedral geometry of the domain has played a major role.\footnote{See e.g.~\citet{garber_faster_2015,garber_linearly_2016,lacoste-julien_global_2015,beck_linearly_2017,pena_polytope_2019}.} 

A noteworthy aspect of conditional gradient algorithms is that they produce solutions which are \emph{sparse}: each iteration $t$ of the procedure yields a convex combination of at most $t$ atoms. 
This is true of all related algorithms known to the authors. 
Sparsity has made conditional gradient algorithms particularly useful in problems such as statistical estimation, signal recovery, and the construction of coresets \citep{li_mixture_2000,zhang_sequential_2003,zhang_boosting_2005,clarkson_coresets_2010,tropp_greed_2004}.

It is therefore reasonable to expect that sparsity might place additional restrictions on the convergence rate, and this is the direction we pursue here.\footnote{It may be desirable to replace sparsity with the number of calls to a linear maximization oracle; we unfortunately do not know how to do this. However, since each call to a linear oracle will output a single atom, it would require some creativity to construct a solution using $\omega(t)$ atoms from $t$ oracle queries.}
In fact, all existing lower bounds for Frank-Wolfe make essential use of sparsity \citep{jaggi_revisiting_2013,garber_linearly_2016,mirrokni_tight_2017,Lan2020}. 
However, instead of constructing a single, hard instance, as in prior work,\footnote{All papers known to the authors construct a hard instance using the simplex. \citet{mirrokni_tight_2017} also derive a sharper lower bound from a random polytope with vertices drawn from the hypercube.} here we derive \emph{domain specific} lower bounds that can be adapted to a wide variety of settings using geometric data. 

At the core of the paper is a basic and well-known geometric idea: if every point in the domain $D$ can be approximated by a sparse set of atoms, then the metric entropy of $D$ cannot be too large. Running the argument in reverse, we see that if the metric entropy of $D$ is large then it takes many atoms\textemdash hence many conditional gradient iterations\textemdash to produce a good approximate solution. 

This approach is simultaneously flexible enough to handle a wide variety of settings\textemdash and even to conduct an average case analysis\textemdash yet it is powerful enough to produce tight results: in particular, to show that in many real-world settings the best dimension-free guarantee possible is $\tilde O(1/\eps)$. We demonstrate this to be true for the $\ell^1(d)$-ball, for the nuclear norm ball, and for a spherical or Gaussian random polytope with a polynomial-in-$d$ number of vertices (with high probability).
Moreover, these results each follow from a single, generic argument using basic information: the volume of a polytope and its number of vertices. 
\begin{table}
  \centering
  \begin{tabular}{|c|c|c|c|c|}
    \hline
 Domain &  Dimension & \# Vertices & Lower Bound & Upper Bound \\ \hline \hline
 Random polytope & $d$ & $\mr{poly}(d)$ & $d\log(1/\eps)$*& ? \\ \hline \hline
 Probability simplex $\Delta(d)$ & $d-1$ & $d$ & $d\log(1/\eps)$& $d\log(1/\eps)$ \\ \hline
 $\ell^1(d)$ Ball & $d$ & $2d$ & $d\log(1/\eps)$ & $d\log(1/\eps)$  \\ \hline
 Nuclear-norm Ball & $m \times n$ &$ \infty$ & $\frac{mn\log(1/\eps)}{(m+n)}$* & $1/\eps$ \\ \hline \hline
 $\ell^\infty(d)$ Ball & $d$ & $2^d$ & $\log(d/\eps)$*  & $d^2\log(1/\eps)$ \\ \hline
 Strongly convex set & $d$ & $\infty$ & $1$ & $1/\sqrt{\eps}$ \\ \hline
  \end{tabular}
\vspace{1em}
  \caption{\label{tab:overview} Lower bounds (up to universal constants and logarithmic factors) for algorithms whose $t^\th$ iterate is a convex combination of at most $t$ vertices. The starred lower bounds are contributions of this paper; the first three lower bounds should be understood to hold only for $\eps < 1/d$, and the fourth for $\eps < 1/(m \wedge n)$.}
\end{table}
\subsection{Discussion of Results}
 Table \ref{tab:overview} summarizes our results on lower bounds for Frank-Wolfe in various settings. In the first four examples, the best possible dimension-free rate is $\tilde O(1/\eps)$, and the upper and lower bounds match in the regime $\eps = \Theta(1/d)$ (or $\eps = \Theta(1/(m \wedge n))$ in the matrix case). In the fifth example ($\ell^\infty(d)$-ball), the same is true, but the bounds match in the regime $\eps = \Theta(1/\log(d))$, which is only relevant for extreme values of $d$. The upper bounds are quoted from the works of \citet{jaggi_revisiting_2013}, \citet{garber_faster_2015}, \citet{garber_linearly_2016}, and \citet{lacoste-julien_global_2015}.


Our results also highlight necessary conditions for accelerated dimension-free rates: the set of extreme points must be roughly ``as complex'' as the entire domain. This holds in particular for strictly convex sets, underscoring the open problem posed by \citet{garber_faster_2015} of establishing or refuting linear convergence in strongly convex domains. 

Finally, we show that improved dimension-free guarantees for Frank-Wolfe (or its variants) can be used to derive fast rates in dictionary aggregation problems. Most notably, we show that in the convex aggregation problem the distance of the risk minimizer to the boundary plays a role analogous to the ``margin condition'' in classification, allowing one to establish a fast rate of convergence in excess risk \citep{tsybakov2003optimal,lecue2007optimal}. Our argument is similar to \citet{zhang_boosting_2005} in that it relies essentially on early stopping, but the outcome is quite novel.
\section{Preliminaries} 
We now introduce the central objects studied in the paper.
\begin{definition} Given a subset $T$ of a Hilbert space $H$, the \emph{$k$-convex hull} is given by 
  \[\conv_k(T) = \Set[\sum_{i=1}^k \lambda_it_i]{t_1,\ldots,t_k \in T, \lambda \in \Delta(k)},\] 
  where $\Delta(k)$ denotes the probability simplex in $\bb{R}^k$. The regular (closed) convex hull is denoted $\conv(T)$.
\end{definition}
This motivates our definition of a \emph{sparse programming algorithm}, namely an optimization algorithm over $\conv(T)$ whose outputs take values in $\conv_k(T)$.
\begin{definition}
  A \emph{$k$-sparse programming algorithm} is a function $A(T,f)$ which, given access to a cost function $f$ and a set $T$,
  outputs some $z \in \conv_k(T)$. The set of all such functions is denoted $\bb{A}(k)$
\end{definition}
We do not restrict the computational complexity of the candidate function in any way, nor the manner in which it depends on $T$ and $f$, as that is irrelevant for our results.
By definition, the $k^\th$ iterate of Frank-Wolfe and all of its usual variants comprises a $k$-sparse programming algorithm\textemdash including the ``away steps'' and ``fully corrective'' variants, as well as the local oracle construction of \citet{garber_linearly_2016}. 
Since we ignore all properties other than sparsity, the central quantity we study is the following.
\begin{definition}
  For a given error tolerance $\eps > 0$, the \emph{sparse programming complexity} $S_F(T,\eps)$ of a set $T$, relative to a class $F$ of cost functions $H \to \bb{R}$, is the smallest integer $k$ such that
  \[ \inf_{A \in \bb{A}(k)}\sup_{f \in F \atop x \in \conv(T)} f(A(T,f)) - f(x) \le \eps.\]
\end{definition}
 We always consider the distance-squared cost functions \[F(H) = \Set[f_x: y \mapsto \norm{x-y}^2]{x \in H},\] and suppress the subscript $F$. Note that this class of cost functions is somewhat ideal in that all its elements are $2$-smooth and $2$-strongly convex. This serves to isolate the dependence of our lower bounds on the geometry of the domain alone.

The starting point of our work is the relation between the sparse programming complexity and the classical \emph{approximate Carath\'eodory} problem in convex geometry, which aims to bound the following quantity. 
\begin{definition}
  The \emph{compressibility} of a set $T \subset H$ is given by 
  \begin{equation}\label{def:compress}
    c(T,\eps)=\inf\Set[k \in \bb N]{\conv(T) \subset \conv_k(T) + \epsilon B(H)},
  \end{equation}
  where $B(H)$ denotes the unit ball in $H$.
\end{definition}
In other words, these numbers quantify how many points of $T$ are necessary to $\eps$-approximate the convex hull. 
Since finding $y$ such that $x \in y + \eps B(H)$ is equivalent to finding an $\eps^2$-approximate minimizer of the cost function $z \mapsto \norm{x-z}^2$, which belongs to $F_H$, we deduce the following.
\begin{lemma} For any $T \subset H$, the compressibility satisfies
  \begin{equation}\label{eq:comp-sparse}c(T, \eps) \le S(T, \eps^2).\end{equation}
\end{lemma}
This relation was noted by \cite{combettes_revisiting_2021}, who showed that modern convergence guarantees for Frank-Wolfe can be used to improve estimates on compressibility. In other words, a convergence guarantee for any sparse programming algorithm provides a \emph{deterministic proof} of low compressibility.

In this work, we pursue a similar line of reasoning using a classical idea often attributed to Maurey that relates
the compressibility to the metric entropy \citep{SAF_1980-1981____A5_0}. 
The metric entropy may be defined as follows.
\begin{definition} The \emph{covering number} $N(\eps,T,d)$ of a set $T \subset H$ with respect to a metric $d$ is defined as the smallest cardinality of a set $U$ such that
  \[\sup_{x \in T} \inf_{y \in U} d(x,y) \le \eps.\] The \emph{metric entropy} is the natural logarithm of the covering number, $\log N(\eps,T,d)$.
\end{definition}
Unless otherwise indicated, we take $d$ to be the norm in $H$ and omit this argument. It is straightforward to show that the metric entropy of $\conv_k(T)$ satisfies the following bound, which forms the basis for our results.
\begin{lemma}\label{lem:khull-informal} Let $T$ be a subset of the unit ball in $H$. Then
  \begin{equation} \label{eq:khull-informal}
  \log N(\eps, \conv_k(T)) \le k(\log N(\eps/2, T) + \log(6/\eps)).
\end{equation}
\end{lemma}
We observe that if the compressibility satisfies $c(T,\eps) \le k$, the metric entropy of $\conv(T)$ cannot be much larger than that of $\conv_k(T)$. By \eqref{eq:khull-informal}, the metric entropy of $\conv_k(T)$ is in turn only roughly $k$ times the metric entropy of $T$. 

Thus, a lower bound on the metric entropy of $\conv(T)$ relative to that of $T$ can be translated into a domain-specific lower bound on the compressibility. By \eqref{eq:comp-sparse}, we recover a lower bound on the sparse programming complexity. This is the essence of Proposition \ref{prop:basic-polytope} and Theorem \ref{thm:infinite-vertex} below.
\begin{rmk} \emph{Carath\'eodory's theorem} \citep[Theorem A.1.3]{artstein-avidan_asymptotic_2015}, a classical result in convex geometry, implies that for $T \subset \bb{R}^d$,
\[d+1 \ge c(0,T) \ge c(\eps/2,T).\] Thus, the above argument cannot give a non-trivial lower bound on the iteration complexity that is larger than the ambient dimension, $d$. It is therefore quite surprising that it recovers tight bounds on the iteration complexity of the form $\tilde \Omega(d\log(1/\eps))$. 
\end{rmk}

\subsection{Notation}
Throughout the paper, the notation $f \lesssim g$ will be used to denote an inequality that holds up to some universal positive constant, and $f(n) \ll g(n)$ will be used interchangeably with $f(n) = o(g(n))$. Similarly, $C$, $C'$, $C''$, etc.~will denote a placeholder for a sufficiently large positive constant, which may not be the same across displays. A glossary of these and other notational conventions may be found in Appendix \ref{sec:nota}.
\section{Lower Bounds}
In view of the above discussion, bounds on the metric entropy can be used to place restrictions on the convergence rate of any sparse programming algorithm. This section explores various applications of this observation, both for constructing general-purpose bounds and for studying specific examples. Omitted proofs are contained in Appendix \ref{sec:lb}. We begin with the following result.
\begin{proposition} \label{prop:basic-polytope}
  Let $P$ be a polytope in $\bb{R}^d$ with $n$ vertices, of unit diameter, whose $\ell^2(d)$ covering numbers are denoted $N(P,\eps).$
  Then, for any $\eps \le 1$ the sparse programming complexity for quadratic objectives over $P$ satisfies
  \[S(P,\epsilon^2/4) \ge \frac{\log N(P,\eps)}{3 + \log n + \log(1/\epsilon)} .\]
\end{proposition}
\begin{proof}
Without any loss of generality, we may assume that $P$ is a subset of the unit ball. Suppose that there exists a sparse programming algorithm which, for all $p \in P$, converges to within a tolerance of $(\epsilon/2)^2$ for the objective $\norm{x-p}^2$ in fewer than $t$ iterations. 
Then, for each $p \in P$ there exist $t$ vertices $(x_1(p), \ldots, x_t(p))$ and $\lambda \in \Delta(t)$ with the property that 
  \[\norm{p - \sum_{i=1}^t \lambda_ix_i(p)} \le  \epsilon/2.\] 
  
  Now, let $T$ be a minimal covering of $\Delta^t$ in $\ell^1(d)$ at resolution $\epsilon/2$, and let $[\lambda]$ denote the projection of $\lambda \in \Delta^t$ onto the nearest element of $T$ (breaking ties arbitrarily). It is straightforward to verify the following facts.

  \begin{lemma} \label{lem:l1-weights-lipschitz}
    If $u_i \in \bb{R}^d$ with $\norm{u_i} \le 1$ for all $1 \le i \le t$, and $\lambda, \lambda' \in \bb{R}^t$, then we have 
    \begin{equation}
      \norm{\sum_{i=1}^t (\lambda_i - \lambda'_i)u_i} \le \norm{\lambda-\lambda'}_1. \label{eq:l1-lispschitz}
    \end{equation}
  \end{lemma}
  \begin{lemma}\label{lem:l1-simp-cover}
    The $\ell_1(t)$ covering numbers of the probability simplex $\Delta(t) \subset \bb{R}^t$ satisfy
    \begin{equation}
      N(\Delta(t),\eps, \ell^1(t)) \le \left(\frac{3}{\eps}\right)^{t}
      \label{eq:l1-simp-cover}
    \end{equation}
  \end{lemma}
  By \eqref{eq:l1-lispschitz} and the triangle inequality, we have
  \[\norm{p - \sum_{i=1}^t [\lambda]_ix_i(p)} \le  \epsilon.\] 
  Moreover, the total number of possible values taken by the expression 
  \[\sum_{i=1}^t [\lambda]_ix_i(p)\]
  can be upper bounded as 
  \[ \#T \binom{n}{t}  \le \#T\left(\frac{en}{t}\right)^t \le \left(\frac{6en}{t\epsilon}\right)^t,\] where we use the estimate of $\#T$ given by \eqref{eq:l1-simp-cover}.
  Thus,  
  \begin{align*}
     \log N(P,\epsilon) &\le t(\log n + \log(6e) - \log t + \log(1/\epsilon)) \\
                       &\le t(\log n + \log(6e) + \log(1/\epsilon)).
  \end{align*}
  Rearranging and noting that the choice of sparse programming algorithm was arbitrary, we deduce that 
  \[\frac{\log N(P,\epsilon)}{\log(6e) + \log n + \log(1/\epsilon)} \le S(\eps^2/4).\]
  Noting that $\log(6e) \le 3$, the proof is complete. 
\end{proof}
As a first illustration of our result, let's consider the unit ball in $\ell^1(d)$, which we'll denote by $B_1(d)$. Plugging in the number of vertices and a volumetric estimate of the covering number gives the following lower bound.
\begin{example}[$\ell^1(d)$ ball]  
\[S(B_1(d),\d/d) \ge \frac{d \log(1/4\d)}{3 + \frac{3}{2}\log d +\log(1/2\d)}\]
Thus, when the dimension $d$ is comparable to the inverse error tolerance $1/\eps$, the number of iterations required for any sparse programming algorithm is $\tilde\Omega(1/\eps)$.
\end{example}
\begin{proof}
  By applying a volume argument (Lemma \ref{lem:vol} in appendix) to the inscribed Euclidean ball, which has radius $1/\sqrt{d}$, one obtains 
  \[N(B_1(d),\eps) \ge \left(\frac{1}{\eps \sqrt d}\right)^d.\]
  Taking $\eps = 4\d/\sqrt{d}$ and plugging into Proposition \eqref{prop:basic-polytope} gives us the lower bound. 
\end{proof}
The above example illustrates a general approach for constructing lower bounds for arbitrary polytopes by means of the volume argument.
\begin{theorem}\label{thm:volume}
  Let $P$ be a polytope of diameter at most $1$. Then 
  \begin{equation}\label{eq:volume}
    S(P,\eps^2/4) \ge d \left( \frac{ \log(\varv(P)^{\f 1 d}) + \log(1/\eps)}{3+\log n + \log(1/\eps)} \right),\end{equation}
  where $\varv(P) = {\vol_d(P)\Gamma(1+\textstyle\frac{d}{2})}/{\pi^{\f d 2}}$ is the ratio between $\vol_d(P)$ and the volume of the unit Euclidean ball, $B_2(d)$.
\end{theorem}
Note that $P$ can be isometrically transformed to lie inside the unit ball, hence $\varv$ is at most $1$. By restricting to the affine hull of $P$ we must have $n \ge d+1$ or else the bound is vacuous. Meanwhile, for the naive implementation of Frank-Wolfe to be tractable in moderate to high dimension, we must have $n \le \mathrm{poly}(d)$, so the first summand of the denominator would be at most $O(\log d)$. 

We show subsequently that a random polytope with $\mathrm{poly}(d)$ vertices has $\varv^{\f 1 d} \ge \Omega(\nicefrac{1}{2\sqrt{d}})$, so our lower bound for the $\ell^1(d)$ ball applies in the average case, up to constants. Thus, it appears that the ``curse of dimensionality,'' namely the necessity for $\tilde\Omega(d)$ iterations to achieve an error tolerance of order $1/d$, is a somewhat generic phenomenon. 

\begin{rmk}
  In a hypothetical scenario where $\varv(P)^{\f 1 d} = \Omega(1)$ and $n \le \mathrm{poly}(d)$, our result \eqref{eq:volume} would imply that $d/\log(d)$ iterations are required to achieve a \emph{constant} error tolerance $\eps < \varv(P)^{\frac 1 d}$, which would contradict the known $O(1/\eps)$ dimension-free iteration complexity of Frank-Wolfe \citep[Theorem 1]{jaggi_revisiting_2013}. This implies that any sequence of polytopes $d \mapsto P_d \subset \bb{R}^d$ containing an $\Omega(1)$ fraction of their circumscribed spheres cannot have a polynomial-in-$d$ number of vertices. 
  
  Indeed, Maurey's lemma (see \citet[Lemma 2]{SAF_1980-1981____A5_0}), which may be derived as a consequence of Frank-Wolfe \citep{mirrokni_tight_2017,combettes_revisiting_2021}, tells us that a polytope $P$ of unit diameter on $n$ vertices has entropy at most $C\eps^{-2}\log(n)$, while $\varv(P)^{\f 1 d} \ge c$ implies the entropy is at least $cd\log(1/\eps)$ (by the volume argument). Choosing $\eps = 1/e$ verifies that $n \ge e^{cd/(Ce^2)}$.
\end{rmk}

We complete our study of finite polytopes with an average-case lower bound showing that the inequality $\varv^{\f 1 d} \gtrsim 1/\sqrt{d}$ satisfied by the $\ell^1(d)$ ball is generically satisfied by a polytope with polynomially many vertices. 

\begin{theorem}\label{thm:random-polytope}
  Let $v = (v_1,\,v_2,\ldots, v_m)$ be $m$ points independently distributed according to the uniform measure on the sphere $S^{d-1} = \Set[x \in \bb{R}^d]{\norm{x}=1}$, with $d \ge 6$. Let $P = \conv(v)$ denote their convex hull. Then there exists a universal $C$ such that if $\alpha d^\alpha \ge m \ge Cd\log^2(d)\log(1/\eta)$ for some $\a > 0$, it holds with probability $1-\eta$ that
  \[\varv(P) \ge \left(\frac{1}{2\sqrt{d}}\right)^d, \quad S(P,\d/d) \ge \frac{d\log(1/\d)}{C(\a + \f 1 2)\log d + \log(1/\d)}.\]
  Moreover, for any particular choice of $m$, the two statements hold with probability at least $1-2\exp(d\log(m)-m/20)$.
\end{theorem}
\begin{proof} We will provide a proof for the spherical case stated above; the promised extension to Gaussian random polytopes is postponed to the Appendix \ref{sec:gauss-poly}. The basis of the proof is a bound due to \citet{dyer1992volumes} and adapted to our context by \citet{pivovarov2007volume}, who also gives an analogous result for the Gaussian case.   While the full proof can be found in the referenced paper, a sketch of the argument is provided in Appendix \ref{sec:pivovarov2007volume-proof} for the reader's convenience. 
  \begin{lemma}[{\citet[Lemma 2.12]{pivovarov2007volume}}] \label{lem:pivovarov2007volume}
    Let $B_2(d)$ denote the unit ball in $\ell^2(d)$, and suppose that $d \ge 6$. Then
    \begin{equation}\label{eq:pivovarov2007volume}\bb{P}\left(rB_2(d) \not\subset P \right) \le 2\exp(d\log(m) - n\mu(r)),\end{equation} where $\mu(r)$ is the measure of the spherical cap of height $r$, namely $\mu(r)=\bb{P}(\bk{v_1}{u} \ge r)$ for a unit vector $u$. 
  \end{lemma}
  
  In order to complete the argument, we need a good lower bound on the quantity $\mu(\nicefrac{1}{2\sqrt{d}})$ so that the probability bound in \eqref{eq:pivovarov2007volume} converges to $0$. Fortunately, we are able to show the following. 
  \begin{lemma}\label{lem:bigcap} For all $d \ge 6$,
    \[\mu\left(\frac{1}{2\sqrt{d}}\right) > \frac{1}{20}.\]
  \end{lemma}
  The bound follows from a more general lower bound on the size of spherical caps (Proposition \ref{prop:big-cap-lb} in the appendix), which uses a coupling of the high-dimensional spherical and Gaussian distributions. Armed with this bound, we have 
  \[\bb{P}\left(\frac{B_2(p)}{2\sqrt{d}} \not\subset P \right) \le 2\exp(d\log(m) - m/20) \le \eta,\] as long as $m \ge Cd\log^2(d)\log(1/\eta)$ for some large enough universal constant $C$. Finally, we note that on the complimentary event, we have 
  \[P \supset \frac{B_2(p)}{2\sqrt{d}} 
  \implies 
  \varv(P) \ge \varv\left(\frac{B_2(p)}{2\sqrt{d}}\right) 
  = \left(\frac{1}{2\sqrt{d}} \right)^d,
  \] which is what we aimed to show. On this event, we may lower bound the sparse programming complexity by plugging the above volume estimate into Theorem \ref{thm:volume} with $\eps = \sqrt{\d/(4d)}$. Noting that $\log(n) \le \a \log(d) + \log(\a)$, we find that
  \[S\left(P, \f{\d}{16d}\right) \ge \frac{\f 1 2 d\log(1/\d)}{3 + \a\log d + \log \a + \log(\sqrt{4d/\d})}.\] Since we have assumed $d \ge 6 \ge e$, the extra terms in the denominator may be absorbed into an appropriate universal constant factor $C$ adjoining $(\a + \f 1 2) \log(d)$, yielding the result.
\end{proof}
\begin{rmk}
  It is evident from the proof of Lemma \ref{lem:pivovarov2007volume} that the result can be extended to any independent random points $v_i \in \bb{R^d}$ each satisfying 
  \[\inf_{\norm{u}=1} \bb{P}\left(\bk{v_i}{u} \ge \frac{\d}{\sqrt{d}}\right) \ge \d\]
  for some $\d > 0$, i.e.~the vertices cannot be too concentrated near any linear subspace. 
\end{rmk}

\subsection*{Large or infinite vertex sets}
If the domain $P$ admits a more sophisticated linear programming oracle\textemdash for example, in cases where a self-concordant barrier can be efficiently computed\textemdash then it is possible to dispense with the requirement that $n = \mr{poly}(d)$, which improves the situation considerably.

As an example in which our lower bounds begin to fail, we consider the suitably normalized unit cube $B_\infty(d)/\sqrt{d} = [\nicefrac{-1}{\sqrt d},\nicefrac{1}{\sqrt d}]^d$, and rescale the error tolerance accordingly.

\begin{example}[Unit cube] \label{ex:unit-cube}
  \[S(B_\infty(d),d\eps^2/4) \ge \frac{d(\log(1/\eps) - C)}{3 + d\log 2 + \log(1/\eps)}\] 
  In particular, taking $\epsilon = e^{-C}\sqrt{2\d/(d\log d)}$, we find that $\Omega(\log(d/\delta))$ iterations are required to achieve an error tolerance of $e^{-2C}\delta/\log(d)$. 
\end{example}

\begin{rmk} Here, the number of iterations required to achieve a \emph{constant} error tolerance $\eps$ is $\Omega(\log d)$. This would seem to contradict an $O(1/\eps)$ bound, but reflects the fact that the diameter of the domain\textemdash in this case $\sqrt d$\textemdash is hidden in the constant of the usual bound (see e.g.~\citet{jaggi_revisiting_2013}).
\end{rmk}

In this case the dimension must be exponential in the number of iterations for an $O(1/\eps)$ upper bound to be tight. For more reasonable values of $\eps$ and $d$, the lower bound behaves like $\log(1/\eps)$. We therefore cannot rule out an $O(d\log(1/\eps))$ iteration complexity upper bound for conditional gradient on the cube, nor even $o(d \log(1/\eps))$ for reasonable values of $d$ and $\eps$.

This differs significantly from the best known upper bounds for conditional gradient procedures on the unit cube, which take the form $O(d^2\log(1/\eps))$ \citep{lacoste-julien_global_2015,garber_linearly_2016,pena_polytope_2019}. 
In fact, as we remarked in the introduction, the sparse programming complexity in dimension $d$ can never be greater than $(d+1)$ for any $\eps$, by Carath\'eodory's theorem.

These considerations extend naturally to other polytopes with exponentially many vertices that arise in combinatorial optimization such as matroid polytopes, flow polytopes, the Birkhoff polytope, and the $k$-marginal polytope on an $n$-clique.
\subsubsection*{Infinite vertex sets}
It is also possible to consider optimization over convex sets $S = \conv(V)$ for  infinite sets $V$, provided one replaces the cardinality $n = \#V$ with the covering number at scale $\epsilon$. The proof follows along the same lines as Proposition \ref{prop:basic-polytope}, and is deferred to the Appendix \ref{proof:infinite-vertex}. 
\begin{theorem}
\label{thm:infinite-vertex}
    Let $V \subset \bb{R}^d$ be a subset of the unit Euclidean ball.
    Then, for any $\eps \le 1$ the sparse programming complexity for quadratic objectives over $\conv(V)$ satisfies
    \[S(V,\epsilon^2/4) \ge \frac{\log N(\conv(V),\eps)}{4 + \log N(V,\eps/2) + \log(1/\epsilon)} .\]
\end{theorem}
A nice application of Theorem \ref{thm:infinite-vertex} is given by the nuclear norm ball, which is another setting where Frank-Wolfe finds frequent use and where a sharp lower bound may be derived from the entropy calculation. 
\begin{example}[Nuclear norm ball]\label{ex:nuclear}
  For $1 \le m < n$, let $B_1(n,m) \subset \bb{R}^{m \times n}$ denote the set of $m \times n$ matrices $A$ whose singular values $\sigma_i(A)$ satisfy
  \[\norm{A}_{S^1} = \sum_{i=1}^m |\sigma_i(A)| \le 1.\] Then \citet[Theorem 1]{1384521} showed that
  \[B_1(n,m) = \conv(V), \quad V = \Set[uv^\top]{u \in \bb{R}^m, v \in \bb{R}^n, \norm{u}_{\bb{R}^m} = \norm{v}_{\bb{R}^n} = 1}.\]
  In this context, we have by Theorem \ref{thm:infinite-vertex} that, whenever $m \ge e$,
  \[S(V, \eps^2/4) \ge \frac{mn\left(\log(1/\eps) - \log(m \wedge n)/2\right)}{6 + m + n + 2\log(1/\epsilon)}\]
  Taking $\eps = \sqrt{\d/(m\wedge n)}$, we find that $mn\log(1/\d)/[C(m+n)]$ iterations are required to achieve an error tolerance of $\d/[4(m\wedge n)]$, giving a sharp $\Omega(1/\d)$  lower bound when $m = \Theta(n) = \Theta(1/\d)$.
\end{example}
\begin{rmk}
  A corresponding linear rate for the nuclear norm ball cannot be derived from results known to the authors, since the set of extreme points is not discrete. Thus, sharpness of the $mn\log(1/\d)/(m+n)$ lower bound remains unresolved.
\end{rmk}
\subsubsection*{Strictly convex sets}
As the most extreme example in which our lower bound fails, let $V = \partial B$ denote the boundary of a closed, strictly convex set $B$. In this case, the entropy of $V$ and $B$ are equal up to constants and an additive factor $\log(1/\eps)$, and our lower bound is at most constant, for any value of $\epsilon$. 

This underscores the open problem highlighted by \citet{garber_faster_2015} of determining whether dimension-independent linear rates (or linear rates of any kind) are attainable for conditional gradient algorithms over strongly convex sets.

\begin{lemma}
  Let $B$ be a strictly convex subset of the unit ball, namely $B = \Set[x]{\vp(x) \le 0}$ for some function $\vp$ which satisfies
  $\vp(\lambda x + (1-\lambda)y) < \lambda\vp(x) + (1-\lambda)\vp(y)$ for $\lambda \in (0,1)$, and let $V=\partial B = \Set[x]{\vp(x)=0}$ denote the boundary. Then 
  \[\log N(B,2\epsilon) \le 2(\log N(V,\epsilon/2) + \log(6/\epsilon))\]
\end{lemma}
\begin{proof}
  Let $u$ be an arbitrary unit vector and put $\psi_x(\lambda) = x + \lambda u$. Then, $\Set[\lambda]{\psi_x(\lambda) \in B}$ is a closed and bounded interval $[l_x, r_x] \subset \bb{R}$, and we can easily verify using strict convexity that $\psi_x(l_x), \psi_x(r_x) \in \partial B$. It follows that $B \subset \conv_2(\partial B)$, since $x$ is a convex combination of $\psi_x(l_x)$ and $\psi_x(r_x)$. The proof is then complete by applying Lemma \ref{lem:khull-informal}.
\end{proof}
\section{Stochastic Sparse Programming}
In this section, we'll show that it is possible to extract {statistical} guarantees simply by analyzing the convergence rate of a conditional gradient algorithm. This is done by appealing to the low complexity of the set of conditional gradient iterates.

For simplicity of exposition, we'll illustrate this section's results using the Gaussian sequence model. 
Let's suppose we observe
\[y = \mu^* + g/\sqrt{n},\] where $g$ is a standard Gaussian in $\bb{R}^n$ and $\mu^* \in \bb{R}^n$ is unknown. We will be interested in controlling the excess mean square error relative to some class $F$, namely 
\[\c E(\hat \mu; F) = \norm{\mu^*-\hat \mu}_2^2 -  \inf_{\mu \in F}\norm{\mu^*-\mu}_2^2.\]
We now state a bound that allows us to exploit the sparsity of Frank-Wolfe iterates as a form of regularization, although it applies to other forms of regularization as well. 
\begin{proposition} \label{thm:local-gaussian}
  Let $\hat \mu$ be an $\eps$-approximate minimizer of the empirical risk that takes values in $G \subset F$, let $\bar \mu \in F$ be the minimizer of the true risk, and suppose that $F$ is convex. Then
  \begin{equation}
    \c E(\hat \mu; F) 
    \le \eps + \sup_{\mu \in G - \bar\mu}\left\{\frac{4}{\sqrt n}\bk{g}{\mu} - \norm{\mu}_2^2 \right\} 
    \le \eps + \frac{8}{n}\left(\sup_{\mu \in \bar\mu - G}\bk{\f{\mu}{\norm{\mu}}}{g}\right)^2.
  \end{equation} 
\end{proposition}
\begin{rmk}
The utility of the above bound is that it is a sharp, localized bound, for example capable of recovering the oracle rates in sparse recovery, and yet it only depends on the local complexity of the approximating set $G$ and the approximation quality $\eps$.
\end{rmk}
Suppose we would like to approximate $\mu^*$ using a finite, normalized dictionary,
\[D = \{\mu_1,\mu_2,\ldots,\mu_m\} \subset S^{d-1}\]
To simplify computations, we assume $m \ge e$. Let $\eps(k)$ denote the error obtained by running a sparse programming algorithm for $k$ steps with the objective $x \mapsto \norm{y-x}^2$ (the empirical risk), and note that the output $\hat \mu_k$ of the procedure belongs to the linear span of at most $k$ elements of $D$. Putting $G = \cup_{S \in \binom{D}{k}} \vspan(S)$, we can compute that 
 \begin{align*}
  \bb{E}\sup_{\mu \in  G} \bk{\frac{\mu}{\norm{\mu}}}{g} 
  &=\bb{E}\sup_{S \subset D \atop \#S = k} \sup_{\mu \in \vspan(S)} \bk{\frac{\mu}{\norm{\mu}}}{g}
  \intertext{Note that the inner supremum is over a subset of a norm ball in a hyperplane of dimension $k$, so it is $1$-subgaussian (Lemma \ref{lem:borell-bk}) with expectation at most $\sqrt{k}$. Thus we have}
  &\le \sqrt{k} + \bb{E}\sup_{S \subset D \atop \#S = k} \left((1-\bb{E})\sup_{\mu \in \vspan(S)} \bk{\frac{\mu}{\norm{\mu}}}{g}\right) \lesssim \sqrt{k\log(m)}
\end{align*}
by the standard sub-Gaussian maximal inequality (Lemma \ref{lem:subg-max}). We obtain the following
\begin{proposition}\label{prop:ksparse-ub}
  Let $\hat \mu_k$ denote the output of running a sparse programming algorithm for $k$ steps to minimize the empirical risk in some convex class $F \subset \vspan(D)$, and let $\eps(k)$ denote the optimization error. Then with probability at least $1-\eta$,
  \begin{equation}
    \c E(\hat \mu_k; F) \lesssim \eps(k) + \frac{k\log(m)}{n} + \frac{\sqrt{k\log(m)\log(1/\eta)^2}}{n}.
  \end{equation}
\end{proposition}
\subsection*{Convex Aggregation}
 If we take $F = \conv(D)$, it is straightforward to show that 
\begin{align*}
  \bb{E}\sup_{\mu \in G - \bar\mu}\left\{\frac{4}{\sqrt n}\bk{g}{\mu}\right\}
  &\le 4\sqrt{\log(m)/n},
\end{align*}
 which matches the minimax rate in the convex aggregation problem, attained by the exact empirical risk minimizer $\hat\mu$ \citep{tsybakov2003optimal}. 
Now, if the empirical risk minimization problem is ``well-conditioned'' for Frank-Wolfe, so that with high probability \[\eps(k) \ll \sqrt{\log(m)/n} \ll 1/k\] for some $k(n)$, then we may improve upon the minimax rate. For one concrete example, suppose $\bar\mu$ belongs to the $n^{\frac{\a-1}{2}}$-relative interior of $\conv(D)$, for some $\a \in (\f 1 2,1]$. We can then appeal to the following result adapted from \citet{garber_faster_2015}.
\begin{theorem}\label{thm:hazan-relint}
  Suppose $y$ belongs to the $r$-relative interior of $\conv(D)$. Then the sub-optimality of the $k^\th$ iterate of Frank-Wolfe for the objective $x\mapsto \norm{y-x}^2$ satisfies 
  \[\eps(k) \le \left(1-\frac{r^2}{16}\right)^k \le \exp\left\{\frac{-r^2k}{16}\right\}.\] 
\end{theorem}
Applying the result to our context, we obtain the following rate. The crux of the proof is showing that under the same hypotheses, $y$ also belongs to the relative interior with high probability.
\begin{corollary}\label{cor:fast-relint}
  If $\bar\mu$ belongs to the $cn^{\f{\a-1}{2}}$ relative interior of $\conv(D)$ for some $\f 1 2 < \a \le 1$, and $\log(m) \vee \log(1/\eta)^2 \le C\sqrt n$ for a sufficiently large universal constant $C$, then we may choose $k = 32\a n^{{1-\a}}\log(n)/c$ and deduce that with probability $1-2\eta$
  \[\c E(\hat \mu_k) \lesssim \frac{\mr{polylog}(m,n,1/\eta)}{n^\a}.\]
\end{corollary}
\begin{rmk}
Note that the radius of $\conv(D)$ is $\Theta(1)$, so values of $\a \le 1$ are reasonable. Thus, the result reflects an unusual fast rate phenomenon that responds to the centrality of  $\bar \mu$. One explanation for this (and Theorem \ref{thm:hazan-relint}) is that the volume of a high-dimensional polytope is generally concentrated at the boundary. 
\end{rmk}
\begin{rmk} It would be very interesting to design adaptive variants of the procedure considered in this section, using e.g. the duality gap. However, that is beyond the scope of our current work.
\end{rmk}
\section{Conclusion}
In this work, we have studied conditional gradient algorithms from the perspective that they output a convex (or linear) combination of at most $t$ atoms, where $t$ is the number of iterations. We establish that the set of outputs of any such procedure must has metric entropy bounded in terms of the number of iterations. If a convergence guarantee holds, this places restrictions on the entropy of the domain; we use this fact to derive domain-specific and sharp lower bounds in numerous settings of interest, as well as in canonical random domains with high probability. As a secondary application, we show that a dimension-free convergence rate which improves on the general $O(1/\eps)$ bound can be used to establish fast rates in statistical estimation. 

\section{Acknowledgments}
{The author is generously supported by the MIT Jerry A. Hausman Graduate Dissertation Fellowship}

\clearpage 
\bibliographystyle{plainnat}
\bibliography{conditional_gradient}

\clearpage

\appendix

\section{Glossary of Notation\label{sec:nota}}

\begin{tabular}{|cp{0.6\textwidth}|} \hline
  $f \gtrsim g$ & inequality up to a universal constant \\
  $f(n) \gg g(n)$ (or $g(n) = o(f(n)))$ & $f(n)/g(n) \uparrow \infty$ as $n \uparrow \infty$ \\
  $f(n) = O(g(n))$ (or $g(n) = \Omega(f(n))$) & asymptotic inequality up to a universal constant  \\
  $f(n) = \tilde O(g(n))$ (or $g(n) = \tilde \Omega (f(n))$) & asymptotic inequality up to logarithmic factors \\ 
  $C, C', C''$ & large universal constant (value may change across displays) \\ \hline \hline
  $S(T,\eps)$ & sparse programming complexity \\
  $c(T,\eps)$ & compressibility \\
  $\conv_k(T)$ & $k$-convex hull \\
  $\conv(T)$ & convex hull \\
  $\bb{A}(k)$ & set of $k$-sparse algorithms (functions)\\
  \hline \hline
  $B_p(d)$ & unit ball in $\ell^p(d)$ \\
  $B_p(m,n)$ & Schatten $\ell^p$ ball in $\bb{R}^{m \times n}$ \\
  $\norm{-}_F$ & Frobenius norm \\
  $\norm{-}_{S^p}$ & Schatten $\ell^p$ norm \\
  $\Delta(d)$ & Probability simplex in $\bb{R}^d$. \\ \hline

\end{tabular}\\

\section{Section 2 Proofs\label{sec:lb}}
\subsection{Proof of Lemma \ref{lem:khull-informal}}
We claimed that if $T$ is a subset of the unit ball then
\[\log N(\conv_k(T),\eps) \le k(\log N(T,\eps/2) + \log(6/\eps)).\]
\begin{proof}
If $x \in \conv_k(T)$ then we may write
\[x = \sum_{i=1}^k \lambda_i t_i\] for $t_1, t_2, \ldots, t_k \in T$ and $\lambda \in \Delta(k)$ (the $k$-ary probability simplex). 

Now, let $S$ be a minimal covering of $T$ at resolution $\eps/2$ (in $H$ norm), and let $D$ be a minimal covering of $\Delta(k)$ at resolution $\eps/2$ (in $\ell^1(d)$ norm). We may assume that $S$ is a subset of the unit ball in $H$ since the projection onto a closed, convex set cannot increase the distance to any point in that set \citep[Lemma 3.1]{bubeck2015convex}. For each index $i$, let $s_i$ denote an element of $S$ such that $\norm{s_i-t_i} \le \eps/2$, and let $\lambda'$ denote an element of $D$ such that $\norm{\lambda - \lambda'}_1 \le \eps/2$. We have 
\begin{equation*}
  \norm{\sum_{i=1}^k \lambda_i t_i - \sum_{i=1}^k \lambda_i s_i} \le \sum_{i=1}^k \lambda_i \norm{t_i-s_i} \le \eps/2.
\end{equation*}
Similarly, by Lemma \ref{lem:l1-weights-lipschitz} (proved independently in Appendix \ref{proof:l1-weights-lipschitz} below), we have 
\begin{equation*}
  \norm{\sum_{i=1}^k \lambda_i s_i - \sum_{i=1}^k \lambda'_i s_i} \le \norm{\lambda-\lambda'}_1 \le \eps/2.
\end{equation*}
By the triangle inequality, we obtain that 
\[\norm{\sum_{i=1}^k \lambda_i t_i - \sum_{i=1}^k \lambda'_i s_i} \le \eps.\] Note that right-hand sum can take on at most $(\#S)^k(\#D)$ values. We have $\#S \le N(T,\eps/2)$; by Lemma \ref{lem:l1-simp-cover} (proved independently in Appendix \ref{proof:l1-simp-cover} below), $(\#D) \le (6/\eps)^k$. We may therefore estimate
\begin{align*}
  N(\conv_k(T),\eps) 
  &\le (\#S)^k(\#D) \\
  &\le N(T,\eps/2)^k(6/\eps)^k.
\end{align*}
Taking logarithms gives the desired result.
\end{proof}
\section{Section 3 Proofs}
\subsection{Proof of Lemma \ref{lem:l1-weights-lipschitz}\label{proof:l1-weights-lipschitz}}
 We claimed that for $u_i \in \bb{R}^d$ with $\norm{u_i} \le 1$ for all $1 \le i \le t$, and $\lambda, \lambda' \in \bb{R}^t$, then 
 \[\norm{\sum_{i=1}^t (\lambda_i - \lambda'_i)u_i} \le \norm{\lambda-\lambda'}_1. 
 \]
 \begin{proof}
We can verify that
\begin{align*}
  \norm{\sum_{i=1}^t (\lambda_i - \lambda'_i)u_i} 
  &\le 
  \sum_{i=1}^t |\lambda_i - \lambda'_i|\norm{u_i} \\
  &\le  \sum_{i=1}^t |\lambda_i - \lambda'_i| = \norm{\lambda-\lambda'}_1,
\end{align*}
as claimed. 
\end{proof}
\subsection{Proof of Lemma \ref{lem:l1-simp-cover}\label{proof:l1-simp-cover}}
We claimed that for all $\eps \le 1$, the probability simplex $\Delta(t) \subset \bb{R}^t$ satisfies
\[N(\Delta(t),\eps, \ell^1(t)) \le \left(\frac{3}{\eps}\right)^{t}.\]
\begin{proof} To verify this, note that $\Delta(t)$ is contained within the unit $\ell^1(d)$ ball $B_1(t)$. Thus, by Lemma \ref{lem:homothetic-vol}, we have 
\[N(\Delta(t),\eps, \ell^1(t)) \le N(B_1(t),\eps, \ell^1(t)) \le \left(1 + \frac{2}{\eps}\right)^{t} \le\left(\frac{3}{\eps}\right)^{t},\] where the final inequality holds for $\eps \le 1$. 
\end{proof}
\begin{rmk} With a more careful application of the volume argument, one can show the upper bound $(2/\eps)^{t-1}$. However, the two results are equivalent for our purposes.
\end{rmk}

\subsection{Proof of Theorem \ref{thm:volume}}
We claimed that 
\[S(P,\eps^2/4) \ge d \left( \frac{ \log(\varv(P)^{\f 1 d}) + \log(1/\eps)}{3 + \log n + \log(1/\eps)} \right).\]
\begin{proof} By the standard volume argument (see Lemma \ref{lem:vol}), 
  \[N(P,\eps) \ge \frac{\vol_d(P)}{\vol_d(\eps B_2(d))} = \frac{\vol_d(P)}{\eps^d\vol_d(B_2(d))}.\]
  Taking logs gives 
  \[\log N(P,\eps) \ge d\log\left(\frac{\vol_d(P)^{\f 1 d}}{\eps\vol_d(B_2(d))^{\f 1 d}}\right).\]
  Plugging this into Proposition \eqref{prop:basic-polytope} gives 
  \[S(\eps^2/4) \ge d \left( \frac{ \log(\vol_d(P)^{\f 1 d}/\vol_d(B_2(d))^{\f 1 d}) + \log(1/\eps)}{3 + \log n + \log(1/\eps)} \right).\]
  Upon plugging in the definition of $\varv(P)$, the proof is complete. 
\end{proof}

\subsection{Sketch of Lemma \ref{lem:pivovarov2007volume}\label{sec:pivovarov2007volume-proof}}
Here we sketch a proof of the inequality \eqref{eq:pivovarov2007volume}, proved by \citet{pivovarov2007volume}, that if $p = (p_1, p_2, \ldots p_m)$ are $m$ independent points distributed according to the uniform measure on the Euclidean sphere in $\bb{R}^d$ and $P = \conv(p)$ denotes their convex hull, then 
\[\bb{P}(P \not\supset rB_2(d)) \le 2\exp(d\log(m) - n\mu(r))\]
\begin{proof}
Consider the $\binom{m}{d}$ potential facets of $P$, each determined by a subset of $d$ points (since no $d+1$ points are coplanar, almost surely). If $P \not\supset rB_2(d)$, then at least one of these subsets must form a facet that intersects $rB_2(d)$. We may thus bound the probability that a particular subset forms a facet that intersects $rB_2(d)$, and then apply a union bound over all subsets.
  
To this end, suppose that for a given subset of indices $A \subset [m]$ of size $d$, the corresponding points $p_A$ form a facet that intersects $rB_2(p)$. Then both of the half-spaces determined by the affine hull of $p_A$ have probability at most $1-\mu(r)$ with respect to the uniform distribution on the sphere. Moreover, since $p_A$ forms a facet, each of the remaining $n-d$ points must independently lie in the same one of these two half-spaces. This event can therefore have probability at most $2(1-\mu(r))^{m-d}$, conditional upon any realization of $p_A$. 
By a union bound over all possible subsets $A$, we get 
  \[\bb{P}\left(rB_2(p) \not\subset P \right) \le 2\binom{m}{d}(1-\mu(r))^{m-d} \le 2\exp(d\log(m) - n\mu(r)),\] where the latter inequality follows from standard estimates such as Stirling's approximation. This completes our sketch of Lemma \ref{lem:pivovarov2007volume}.
\end{proof}

\begin{rmk} It is quite clear that this proof can be relaxed to arbitrary independent distributions of the points $p_i$, where $\mu(r)$ is replaced by any uniform upper bound on the probability that each point lies on either side of any hyperplane that intersects $rB_2(d)$. For example, in the Gaussian case, this quantity is precisely $1 - \Phi(r)$, where $\Phi$ denotes the Gaussian CDF.
\end{rmk}
\subsection{Proof of Lemma \ref{lem:bigcap}}
The claimed result was that for all $d \ge 6$
\[\mu\left(\frac{1}{2\sqrt{d}}\right) > \frac{1}{20}.\]
This follows from the following lower bound on the volume of a spherical cap of height $\ep/(1+t)$, which is particularly useful in the regime where $\ep = \ep(d) \downarrow 0$, and $t$ is a small constant, e.g.~$1$. Its proof uses the well-known coupling between the high dimensional spherical and Gaussian distributions.

\begin{proposition}\label{prop:big-cap-lb} Let $\mu(r) = \bb{P}(\bk{v}{e_1} \ge r)$ for $v$ uniformly distributed on $S^{d-1}$ and $e_1$ a canonical basis vector in $\bb{R}^d$. Let $\Phi(-)$ denote the Gaussian CDF. For all $\eps \in (0,1)$ and $t > 0$, it holds that
 \[\mu((1+t)^{-1}\ep) \ge \Phi(\ep\sqrt d) - e^{\frac{-t^2d}{2}}\]
\end{proposition}
Starting from this proposition, we may take $t = 1$, $\ep = \f{1}{\sqrt d}$, and use the lower bound $\Phi(t) \ge (t\sqrt{8\pi})^{-1} e^{-t^2/2}$ to obtain 
\[\mu\left(\frac{1}{2\sqrt{d}}\right) \ge \frac{1}{\sqrt{8\pi e}} - e^{-d/2}\]
For $d \ge 6$, this quantity is at least $1/20$.

\begin{proof}
  Let $g$ be a standard Gaussian random vector in $\bb{R}^d$. Since $g/\norm{g}$ is uniformly distributed on the sphere $S^{d-1}$, our problem is equivalent to stating a lower bound for the probability of the event
  \[E_1 = \left\{\bk{\frac{g}{\norm g}}{e_1} \ge \frac{\ep}{1+t}\right\},\] where $e_1$ is a canonical basis vector. 
  We start by considering the alternative event
  \[E_2 = \left\{\bk{\frac{g}{\sqrt d}}{e_1} \ge \ep\right\}.\]
  Since the  $\bk{g}{e_1}$ is a unit normal random variable, we have that $\bb{P}(E_2) = \Phi(\ep\sqrt{d})$. Now, we have 
  \[\bb{P}(E_1) \ge \bb{P}(E_2 \setminus \bar E_1) = \bb{P}(E_2) - \bb{P}(E_2 \cap \bar E_1),\] so it suffices to bound the probability of $E_2 \cap \bar E_1$. Indeed, rearranging the two corresponding inequalities yields
  \[\norm{g} > (1+t){\sqrt{d}}.\] Note that $(\bb{E}\norm{g})^2 \le \bb{E}\norm{g}^2 = d$ by Jensen's inequality. Since $\norm{g}$ is the supremum of a Gaussian process with pointwise variance $d$ and $\bb{E}\norm{g} \le \sqrt d$, we deduce from Borell's inequality (Theorem \ref{lem:borell}) that
  \begin{equation}\label{eq:borell-gauss-norm}\bb{P}(\norm{g} \ge \sqrt{d} + t\sqrt{d}) \le e^{\frac{-t^2d}{2}}.\end{equation} Thus,
  \[\bb{P}(E_1) \ge \bb{P}(E_2) - \bb{P}(E_2 \cap \bar E_1) \ge \Phi(\ep\sqrt{d}) - e^{\frac{-t^2d}{2}}.\] 
\end{proof}
\begin{figure}[h!]
  \centering
\includegraphics[width=4in]{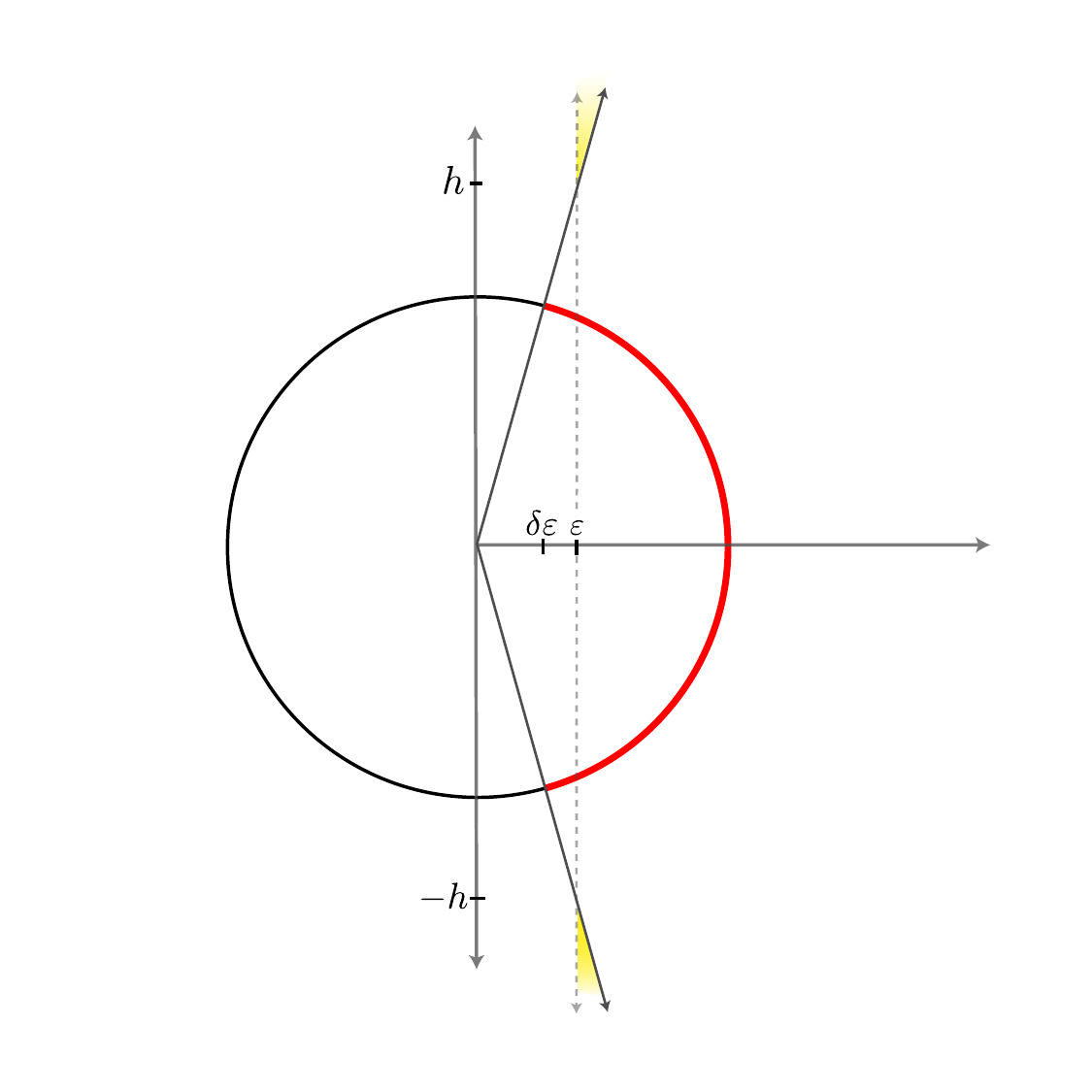}
\caption{Illustration of the proof of Proposition \ref{prop:big-cap-lb}, where $\d = (1+t)^{-1}$. Here, the bold red line depicts the spherical cap of height $\d\ep$. Relative to the Gaussian measure, the event $E_2$ is the region to the right of the dotted gray vertical line, while the event $E_1$ is the cone generated by the red spherical cap. The yellow regions correspond to the event $E_2 \cap \bar E_1$, which implies an upper deviation of size $t\sqrt{d}$ from the mean of the $1$-subgaussian random variable $\norm{g}$.}
\end{figure}
\subsection{Extension to Gaussian Polytopes\label{sec:gauss-poly}}
In this section we'll prove an analogous result to Theorem \ref{thm:random-polytope} where the vertices are independently distributed according to the standard Gaussian measure in $\bb{R}^d$. 
\begin{proposition}
  Let $g = (g_1,\,g_2,\ldots, g_m)$ be $m$ independent standard Gaussian vectors in $\bb{R}^d$, with $d \ge 6$. Let $P' = \conv(g)$ denote their convex hull. Then there exists some $C(\a), C'(\a)$, depending only on $\a$, such that if $\alpha d^\alpha \ge m \ge 10 d\log^2(d)$ for some $\a > 0$, it holds with probability $1-3e^{-d}$ that
  \[\varv(P') \ge \left(\frac{1}{C(\a)\sqrt{d}}\right)^d, \quad S(P'/\sqrt{d},\d/d) \ge \frac{d\log(1/\d)}{C'(\a)\log d + \log(1/\d)}.\]
\end{proposition} 
\begin{proof}
  Let $r > 0$ be given. We use the following result, whose proof mimics Lemma \ref{lem:pivovarov2007volume}.
  \begin{lemma}[{\citet[Lemma 2.7 \& Eqn. 11]{pivovarov2007volume}}] Let $\Phi(r) = \bb{P}(g_1<r)$ denote the Gaussian CDF. Then
    \begin{equation}\label{eq:piv-gauss} \bb{P}\left( P' \not\supset rB_2(d) \right) \le 2\binom{m}{d}\Phi(r)^{m-d} \le 2\exp(n\log(m)-n(1-\Phi(r)))
    \end{equation}
  \end{lemma}
  Working forwards from \eqref{eq:piv-gauss}, 
  since $1 - \Phi(1) \ge \f{1}{10}$, 
  we obtain 
  \[\bb{P}\left( P' \not\supset B_2(d) \right) \le 2\exp\{d\log(m) - m/10\}.\]
  If $m \ge 10d\log^2(d)$ then this event holds with probability at least $2e^{-d}$. 
  
  In order to apply Theorem \eqref{thm:volume}, we also need an appropriate estimate on the radius of $P'$. 
  Note that $\max_{i \in [m]} \norm{g_i} = \max_{i \in [m]} \sup_{\sigma \in S^{d-1}} \bk{\sigma}{g_i}$ is the supremum of a unit variance Gaussian process, so by Borell's inequality (Theorem \ref{lem:borell}) we have 
  \[\bb{P}\left(\max_{i \in [m]} \norm{g_i} \ge \bb{E}\left[\max_{i \in [m]} \norm{g_i}\right] + u\right) \le e^{-u^2/2}\]
  To estimate the expectation, note by \eqref{eq:borell-gauss-norm}, we have $\bb{P}(\norm{g_i} - \sqrt{d} \ge u) \le e^{-u^2/2}$. Thus, by the sub-Gaussian maximal inequality (Lemma \ref{lem:subg-max}), we have that 
  \[\bb{E}\left[\max_{i \in [m]} \norm{g_i} - \sqrt{d}\right] \le \sqrt{2\log(m)} + \sqrt{2\pi}.\]
  Since $\log(m) \le \a\log(d) + \log(\a)$, we obtain that 
  \[\bb{P}\left(\max_{i \in [m]} \norm{g_i} \ge \sqrt{d} + \sqrt{2\a\log(d) + 2\log(\a)} + \sqrt{2\pi} + u\right) \le e^{-u^2/2}\]
  Absorbing negligible terms and simplifying, we obtain
  \[\bb{P}\left(\max_{i \in [m]} \norm{g_i} \ge C(\a)\sqrt{d}\right) \le e^{-d}\] for some sufficiently large $C(\a)$ which depends only on $\a$. By a union bound, it holds with probability $1 - 3e^{-d}$ that \[B_2(d) \supset \frac{P'}{C(\a)\sqrt d} \supset \frac{B_2(d)}{C(\a)\sqrt d} \implies \varv\left(\frac{P'}{C(\a)\sqrt d}\right)^{\f 1 d} \ge \frac{1}{C(\a)\sqrt d}.\]
  Plugging this into Theorem \ref{thm:volume} gives
\[S(P'/(C(\a)\sqrt d), \eps^2/4) \ge d\left(\frac{\log(1/\eps) - \log(C(\a)\sqrt{d})}{3 + \a(\log d + \log \a) + \log(1/\eps)}\right),\] which yields the result when we choose $\eps = \sqrt{\d/(C(\a)^2d)}$, simplify, and note that $S(u A, u^2 r) = S(A, r)$.

\end{proof}
\subsection{Proof of Theorem \ref{thm:infinite-vertex}\label{proof:infinite-vertex}}
We claimed that if $V \subset \bb{R}^d$ is a subset of the unit Euclidean ball,
then for any $\eps \le 1$ the sparse programming complexity for quadratic objectives over $\conv(V)$ satisfies
\[S(V,\epsilon^2/4) \ge \frac{\log N(\conv(V),\eps)}{4 + \log N(V,\eps/2) + \log(1/\epsilon)} .\]
\begin{proof}
  We work forwards from Lemma \ref{lem:khull-informal}, which says that 
  \[\log N(\eps, \conv_k(T)) \le k(\log N(\eps/2, T) + \log(6/\eps)).\]
  Combining this with the definition of compressibility \eqref{def:compress} gives
\begin{align*}
  \log N(\conv(T),\eps) &\le \log N(\conv_{c(\eps/2,S)}(T) + \eps B(H)/2,\eps) \\
                   &\le \log[N(\conv_{c(\eps/2,S)}(T),\eps/2)\cdot N(\eps B(H)/2,\eps/2) \\
                   &\le c(\eps/2,T)\cdot(\log N(\eps/4, T)+ \log(12/\eps) + 1)
\end{align*}
since $N(\eps B(H)/2, \eps/2) = 1 \le c(\eps/2,S)$. Rearranging and applying \eqref{eq:comp-sparse} to relate compressibility to sparse programming complexity, we obtain a lower bound on the sparse programming complexity in terms of the gap between $N(\conv(S),\eps)$ and $N(\eps/2, S)$:
\begin{equation} \label{eq:compressibility-lb}
  \frac{N(\conv(S),\eps)}{\log N(\eps/2, S)+ \log(12/\eps) + 1} \le c(\eps/2,S) \le S(F_H, D, \eps^2/4)..
\end{equation}
The proof is complete after noting that $\log(12/\eps) + 1 = \log(1/\eps) + \log(12e)$, and $\log(12e) \le 4$.
\end{proof}
\subsection{Proof of Example \ref{ex:unit-cube}}
We claimed that 
\[S(B_\infty(d),d\eps^2/4) \ge \frac{d(\log(1/\eps) - C)}{1 + d\log 2 + \log(1/\eps)}.\] 
\begin{proof} To verify the claim, note that the normalized unit $B_\infty(d)/\sqrt{d}$ cube has volume precisely $(1/\sqrt{d})^d$ and has $n = 2^d$ vertices. On the other hand, the unit ball $B_2(d)$ has volume 
\[\frac{\pi^{\f d 2}}{\Gamma(1+\textstyle\frac{d}{2})} \ge \left(\frac{\sqrt{2\pi/C}}{\sqrt d}\right)^d,\] where we have used the standard inequality $\Gamma(1 + x) \le (Cx)^{x}$. It follows that $\varv(B_\infty(d)/\sqrt{d})^{\frac 1 d} \ge 1/C'$. Combining this with Theorem \ref{thm:volume} gives 
\[S(P,\eps^2/4) \ge d \left( \frac{ \log(1/\eps) - C''}{3+d\log(2) + \log(1/\eps)} \right),\] which is precisely the claimed bound.
\end{proof}
\subsection{Proof of Example \ref{ex:nuclear}}
We claimed that whenever $m \ge e$,
  \[S(V,\epsilon^2/4) \ge \frac{mn(\log(1/\eps)-\frac 1 2 \log(m \wedge n)}{6 + m + n + 2\log(1/\eps)}\]
\begin{proof}
  Firstly, note that the standard Euclidean inner product structure in $\bb{R}^{n \times m}$ coincides with the Frobeinus inner product $\bk{A}{B}_F = \tr(A^\top B)$. Recall that we have
  \[V = \Set[uv^\top]{u \in \bb{R}^m, v \in \bb{R}^n, \norm{u}_{\bb{R}^m} = \norm{v}_{\bb{R}^n} = 1}\] We begin by estimating $\log N(V, \eps)$. For unit vectors $a,c \in \bb{R}^m$ and $b,d \in \bb{R}^n$,
  \begin{align*}
    ab^\top - cd^\top 
    &= ab^\top - cb^\top + cb^\top - cd^\top \\
    &= (a-c)b^\top + c(b-d)^\top,
  \end{align*}
and, each of these summands may be bounded as, e.g.,
\begin{align*}
  \norm{(a-c)b^\top}^2_F 
  &= \tr((a-c)b^\top b (a-c)^\top) \\
  &= \tr((a-c)^\top(a-c)) \\
  &= \norm{a-c}_{\bb{R}^m}^2,
\end{align*}
using the fact that $b^\top b = \norm{b}^2_{\bb{R}^n}=1$. Using the triangle inequality, we may conclude in this manner that 
\[\norm{ab^\top - cd^\top}_F \le \norm{a-c}_{\bb{R}^m} + \norm{b-d}_{\bb{R}^n}.\] Now, suppose  we are given a minimal covering $M$ (resp. $N$) of $S^{m-1}$ (resp. $S^{n-1}$) at resolution $\eps/2$. We may assume without loss of generality that these coverings are subsets of the unit ball, since the nearest point map onto the unit ball does not increase the distance to any point in the unit ball \citep[Lemma 3.1]{bubeck2015convex} . By the above computation, the set \[M N^{\top} = \Set[ab^\top]{a \in M, b \in N}\] is a covering of $V$ of size $(\#M)(\#N)$ at resolution $\eps$. By a volume argument (Lemma \ref{lem:homothetic-vol}) we have that $N(S^{d-1},\eps) \le (3/\eps)^d$ for all $\eps \le 1$, since $S^{d-1}$ is a subset of the unit ball. We may therefore conclude that 
\[N(V,\eps)\le (6/\eps)^{m+n}.\]

To estimate $N(B_1(m,n),\eps)$, note that vector $\vec \sigma(A)$ of singular values of $A \in \bb{R}^{m \times n}$ has at most $\mr{rank}(A) \le m \wedge n$ nonzero entries. Thus, if  $\norm{A}_F = \norm{\vec\sigma(A)}_2 \le 1/\sqrt{m \wedge n}$ then $\norm{A}_{S^1} = \norm{\vec\sigma(A)}_1 \le 1$. Appealing to the equivalence between the Frobenius and canonical Euclidean norms, we deduce that 
\[\varv(B_1(m,n)) \ge \left(\frac{1}{\sqrt{m \wedge n}}\right)^{mn}.\] By another volume argument (Lemma \ref{lem:vol}) this implies 
\[\log N(B_1(m,n),\eps) \ge mn\left(\log(1/\eps) - \log(m\wedge n)/2\right).\] Plugging these estimates into Theorem \ref{thm:infinite-vertex} gives
\[S(V,\epsilon^2/4) \ge \frac{mn\left(\log(1/\eps) - \log(m \wedge n)/2\right)}{4 + m + n + \log(6) + 2\log(1/\epsilon)} \ge \frac{mn\left(\log(1/\eps) - \log(m\wedge n)/2\right)}{6 + m + n + 2\log(1/\epsilon)},\] since $\log(6) \le 2$. Finally, taking $\eps = \sqrt{\d/m \wedge n}$ gives us
\[S(V,\d/[4(m \wedge n)]) \ge \frac{mn\log(1/\d)/2}{C + m + n + \f 1 2 \log(m \wedge n) + \f 1 2 \log(1/\d)} \ge \frac{mn\log(1/\d)}{C'(m + n + \log(1/\d))}.\] 
This completes the proof.
\end{proof}
\section{Proofs from Section 4}
\subsection{Proof of Proposition \ref{thm:local-gaussian}}
The result follows from the following, slightly more general statement.
\begin{proposition} 
  Let $\hat \mu$ be an $\eps$-approximate minimizer of the empirical risk that takes values in $G \subset F$, let $\bar\mu \in F$ be the minimizer of the true risk, and suppose that $F$ is convex. Then
  \begin{equation}
    \c E(\hat \mu) 
    \le \eps + \frac{4}{\sqrt n}\sup_{\mu \in G - \bar\mu}\left\{\bk{g}{\mu} - \norm{\mu}_2^2 \right\} 
    \le \eps + \frac 8 n \left(\sup_{\mu \in \bar \mu - G} \bk{\frac{\mu}{\norm \mu}}{g}\right)^2.
  \end{equation} 
  Moreover, under the same conditions, it holds with probability $1 - \eta$ that
  \begin{equation}
    \c E(\hat \mu) \le \eps + \frac{4}{\sqrt n}\sup_{\mu \in G_\eta} \bk{\bar \mu - \mu}{g}
  \end{equation}
  where $G_\eta \subset G$ is given by the localized set
  \[G_\eta=\Set[\mu \in G]{\sqrt n \norm{\bar\mu-\mu} \le 2\bb{E}\sup_{\mu \in G}\bk{\f{\mu}{\norm{\mu}}}{g} + \sqrt{8\log(1/\eta)}}.\]
\end{proposition}
\begin{proof}
  Let $\bar\mu$ denote the nearest point to $\mu^*$ in $F$, so 
\[\c E(\hat \mu) = \norm{\mu^*-\hat \mu}_2^2 -  \norm{\mu^*-\bar \mu}^2.\]
By convexity of $F$ and Lemma \ref{lem:quadratic-margin}, we have 
\[ \norm{\mu^* - \hat \mu}^2 -  \norm{\mu^* -\bar \mu}^2 \ge \norm{\bar \mu - \hat \mu}^2.\]
Combining these two yields
\[\c E(\hat \mu) = 2\left\{\norm{\mu^*-\hat \mu}^2 -  \norm{\mu^*-\bar \mu}^2\right\} -  \norm{\bar \mu - \hat \mu}^2.\]
Finally, since $\hat \mu$ is an $\eps$-approximate minimizer of $\norm{y-\mu}^2$ in $F$, we have
\[\norm{y-\hat \mu}^2 - \norm{y-\bar \mu}^2 \le \eps.\]
Writing the norm as an inner product and rearranging gives
\[\norm{\mu^* - \bar \mu}^2 - \norm{\mu^* - \bar \mu}^2 \le 4\bk{\bar \mu - \hat \mu}{g/\sqrt n} + \eps.\]
Thus, we have
\begin{align*}
  \c E(\hat \mu) &\le \eps + 4\bk{\bar \mu - \hat \mu}{g/\sqrt n} -  \norm{\bar \mu - \hat \mu}^2. \\
  &\le \eps + \sup_{\mu \in \bar \mu - G} \frac{4}{\sqrt n}\bk{\mu}{g} -  \norm{\mu}^2 \\
  &\le \eps + \sup_{\mu \in \bar G} \frac{4}{\sqrt n}\bk{\mu}{g} -  \norm{\mu}^2 \\ 
  &= \eps + \sup_{\mu \in \bar G} \frac{4}{\sqrt n}\bk{\f{\mu}{\norm{\mu}}}{g}\norm{\mu} -  \norm{\mu}^2
\end{align*}
where $\bar G = \mr{star}(0,\bar \mu - G) \supset \bar \mu - G$. Note that since $\bar G$ is star-shaped, the supremum must be attained at some $\mu$ satisfying  
\begin{equation}
  \norm{\mu} \le \frac{2}{\sqrt n}\bk{\f{\mu}{\norm{\mu}}}{g} \le \frac{2}{\sqrt n} \sup_{\mu \in \bar G}\bk{\f{\mu}{\norm{\mu}}}{g} =  \frac{2}{\sqrt n}\sup_{\mu \in \bar\mu - G}\bk{\f{\mu}{\norm{\mu}}}{g}\label{eq:random-radius}
\end{equation}
where the last step follows from linearity of the objective. Plugging this in, we get the first announced bound
\[\c E(\hat \mu) \le \eps + \frac{8}{n}\left(\sup_{\mu \in \bar\mu - G}\bk{\f{\mu}{\norm{\mu}}}{g}\right)^2\]
For the second bound, we put
\begin{equation}
  \rho(G) = \bb{E} \sup_{\mu \in \bar\mu- G}\bk{\f{\mu}{\norm{\mu}}}{g} = \bb{E}\left[\bk{\bar\mu}{g} + \sup_{\mu \in G}\bk{\mu}{-g}\right] = \bb{E}\sup_{\mu \in G}\bk{\f{\mu}{\norm{\mu}}}{g}
\end{equation}
where we have used symmetry of the Gaussian distribution in the final step. By Borell's inequality (Theorem \ref{lem:borell}) we have that with probability $1-\eta$,
\begin{equation}
  \sqrt{n}\norm{\mu} \le 2 \sup_{\mu \in \bar \mu - G}\bk{\f{\mu}{\norm{\mu}}}{g}  \le 2\rho(G) + \sqrt{8\log(1/\eta)}.  \label{eq:fixed-radius}
\end{equation}
Let $G_\eta \subset G$ be the subset of points $\mu$ satisfying \eqref{eq:fixed-radius}. On the event where \eqref{eq:fixed-radius} holds, we have 
\[\c E(\hat \mu) \le \eps + \f{4}{\sqrt n} \sup_{\mu \in G_\eta} \bk{\f{\mu}{\norm{\mu}}}{g}.\]
This completes the proof.
\end{proof}

\subsection{Proof of Proposition \ref{prop:ksparse-ub}}
We claimed that 
\[\c E(\hat \mu_k; F) \lesssim \eps(k) + \frac{k\log(m)}{n} + \frac{\sqrt{k\log(m)\log(1/\eta)^2}}{n}.\]
Proposition \ref{thm:local-gaussian} tells us that 
\begin{equation}\label{eq:ksparse-ub-localgaussian}
  \c E(\hat \mu) \le \eps + \frac{8}{n}\left(\sup_{\mu \in \bar\mu - G}\bk{\f{\mu}{\norm{\mu}}}{g}\right)^2
\end{equation}
Thus, it will be sufficient to state a high probability bound for the quantity 
\[\sup_{\mu \in \bar\mu - G}\bk{\f{\mu}{\norm{\mu}}}{g}.\]
Note that by Borell's inequality (Theorem \ref{lem:borell}), we must have 
\begin{equation}
  \bb{P}\left( \sup_{\mu \in \bar\mu - G}\bk{\f{\mu}{\norm{\mu}}}{g} - \bb{E}\sup_{\mu \in \bar\mu - G}\bk{\f{\mu}{\norm{\mu}}}{g} > u \right) \le e^{u^2/2}. \label{eq:ksparse-ub-tail}
\end{equation}
Moreover, we verified in the main text that 
\[\bb{E}\sup_{\mu \in  G} \bk{\frac{\mu}{\norm{\mu}}}{g} \le \sqrt{k} + \bb{E}\sup_{S \subset D \atop \#S = k} \left((1-\bb{E})\sup_{\mu \in \vspan(S)} \bk{\frac{\mu}{\norm{\mu}}}{g}\right).\]
Note that the second summand is the expectation a maximum of $\binom{m}{k}$ variables $X_i$, each of which satisfies $\bb{P}(X_i > u) \le e^{-u^2/2}$ (by another application of Borell's inequality). Therefore, by Lemma \ref{lem:subg-max}, we have 
\begin{equation}
  \bb{E}\sup_{\mu \in  G} \bk{\frac{\mu}{\norm{\mu}}}{g} \le \sqrt k + \sqrt{2\log{\binom{m}{k}}} + \sqrt{2\pi} \le \sqrt{3k + 6(\pi+k\log m)},
\end{equation} using the standard bounds $\binom m k \le m^k$ and $\sqrt{a} + \sqrt{b} + \sqrt{c} \le \sqrt{3(a + b + c)}$. Combining this with \eqref{eq:ksparse-ub-tail} and inverting the probability bound gives us that with probability at least $1-\eta$ we have  
\[\sup_{\mu \in \bar\mu - G}\bk{\f{\mu}{\norm{\mu}}}{g} \le \sqrt{3k + 6(\pi+k\log m)} + \sqrt{2\log(1/\eta)}.\] Finally, plugging this into \eqref{eq:ksparse-ub-localgaussian}, expanding the square, and absorbing absolute constants gives us the claimed result so long as $m \ge e$. 

\subsection{Proof of Proposition \ref{cor:fast-relint}}
We claimed that if $\bar\mu$ belongs to the $cn^{\f{\a-1}{2}}$ relative interior of $\conv(D)$ for some $\f 1 2 < \a \le 1$, and $\log(m) \vee \log(1/\eta)^2 \le C\sqrt n$ for a sufficiently large universal constant $C$, then for $k = 32\a n^{{1-\a}}\log(n)/c$ and with probability $1-2\eta$
  \[\c E(\hat \mu_k) \lesssim \frac{\mr{polylog}(m,n,1/\eta)}{n^\a}.\]
  \begin{proof}
We begin by verifying that the empirical risk minimizer belongs to the $(c/2)n^{\frac{\a-1}{2}}$-relative interior with high probability. Indeed, let $\hat\mu$ denote the empirical risk minimizer, i.e. the minimizer of $x \mapsto \norm{x-y}^2$ over $\conv(D)$. We have that by Proposition \ref{thm:local-gaussian} with $G = F = \conv(D)$ that
\[\c E(\hat \mu) 
\le 0 + \frac{4}{\sqrt n}\sup_{\mu \in F - \bar\mu}\left\{\bk{g}{\mu} - \norm{\mu}_2^2 \right\}
\le \frac{4}{\sqrt n}\sup_{\mu \in F - \bar\mu}\bk{g}{\mu}.\]
Moreover, we have that
\[\bb{E}\sup_{\mu \in F - \bar\mu}\bk{g}{\mu} = \bb{E}\sup_{\mu \in F}\bk{g}{\mu} - \bb{E}\bk{g}{\bar\mu} = \bb{E}\sup_{\mu \in D}\bk{g}{\mu}\] by symmetry of the Gaussian distribution and the fact that a linear function over $\conv(D)$ will be maximized at some point in $D$. By the sub-Gaussian maximal inequality, this is at most $\sqrt{C\log(m)}$, and combining this with Borell's inequality (Theorem \ref{lem:borell}) we obtain
\[\c E(\hat \mu) = \norm{\hat\mu - \mu^*}^2 - \norm{\bar\mu - \mu^*}^2 \le \sqrt{\frac{C\log m}{n}} + \frac{2\ln(1/\eta)}{\sqrt n}.\]
Now, on this event, since $\hat\mu$ minimizes $x \mapsto \norm{y-x}^2$ over the convex set $F$, we have by Lemma \ref{lem:quadratic-margin} that
\[\norm{\hat\mu - \bar\mu}^2 \le \sqrt{\frac{C\log m}{n}}+\frac{2\ln(1/\eta)}{\sqrt n}.\] 
Thus, on this event we have $\norm{\hat\mu - \bar\mu} \le C'(\log{m}/n)^{\frac{1}{4}} + C'({\log(1/\eta)^2}/n)^{\frac 1 4}$, since $\sqrt{a+b} \le \sqrt{a} + \sqrt{b}$. For the range of $\log{m} \vee \log(1/\eta)^2 \le \sqrt{n}/C''$ we consider, this is upper bounded by $cn^{\frac{\a-1}{2}}/2$ and hence the distance to the boundary is must be at least $cn^{\frac{\a-1}{2}}/2$. Plugging this value into Theorem \ref{thm:hazan-relint} gives us 
\[\eps(k) \le \exp\left\{ \frac{-ck}{32n^{1-\a}}\right\} \le \frac{1}{n^\a},\] when we choose $k = 32\a n^{1-\a}\log(n)/c$.
Combining this with Proposition \ref{prop:ksparse-ub} and a union bound gives us that with probability $1-2\eta$
\[\c E(\hat \mu_k) \le \frac{1}{n^\a} + \frac{C'''\log(n)}{n^\a} + \frac{\sqrt{2C'''\log(n)\log(2/\eta)}}{n^{\f 1 2 + \f \a 2}}\]
which reduces to the claimed bound. 
  \end{proof}
\section{Technical Lemmas}
\subsection{Covering Lemmas}
We make heavy use of the standard volume argument, which is stated for convenience below.
\begin{lemma}[{\citet[Theorem 4.1.13]{artstein-avidan_asymptotic_2015}}]\label{lem:vol}
  Let $L$ be a convex subset of $\bb{R}^d$ and let $\norm{-}_K$ denote a norm in $\bb{R}^d$ with unit ball $K$. Then 
  \begin{equation}
    \frac{\vol_d(L)}{\vol_d(\eps K)} \le N(L,\eps,\norm{-}_K) \le 2^d\frac{\vol_d(L+ \frac{\eps}{2}K)}{\vol_d(\eps K)}
  \end{equation}
\end{lemma}
We also use the following result, which is an easy corollary of Lemma \ref{lem:vol}.
\begin{lemma}[{\citet[Corollary 4.1.15]{artstein-avidan_asymptotic_2015}}]\label{lem:homothetic-vol}
  Let $K \subset \bb{R}^d$ be centrally symmetric convex set, and let $\norm{-}_K$ denote the norm in $\bb{R}^d$ with unit ball $K$. Then 
  \begin{equation}
    \left(\frac{1}{\eps}\right)^d \le N(K,\eps,\norm{-}_K) \le \left(1 + \frac{1}{\eps}\right)^d
  \end{equation}
\end{lemma}
\subsection{Probability Lemmas}
The following concentration result for Gaussian suprema, which is a consequence of the Gaussian isoperimetric inequality, is heavily used throughout the paper. 
\begin{theorem}[Borell's Inequality, {\citet[Theorem 2.1.1]{adler2007random}}]\label{lem:borell}
  Let $(g_t)_{t \in T}$ be a separable and almost surely bounded Gaussian process, and put $\sigma^2 = \sup_{t \in T}\bb{E}g_t^2$. Then,
  \[\bb{P}\left( \sup_{t \in T} g_t > \bb{E} \sup_{t \in T} g_t + u \right) \le e^{\frac{-u^2}{2\sigma^2}}.\]
\end{theorem}
In particular, noting that if $g$ is a standard Gaussian in $\bb{R^d}$ and $T$ a subset of the unit ball, then $t \mapsto \bk{g}{t}$ is a separable and a.s.~bounded Gaussian process indexed by $t$ and $\bb{E}\bk{g}{t}^2 = \norm{t}^2$, we have the following immediate corollary.
\begin{lemma}\label{lem:borell-bk} Let $g$ be a standard Gaussian in $\bb{R^d}$ and let $T$ be a subset of the unit ball in $\bb{R}^d$. Then 
  \[\bb{P}\left( \sup_{t \in T} \bk{g}{t} > \bb{E} \sup_{t \in T} \bk{g}{t} + u \right) \le e^{\frac{-u^2}{2}}.\]
\end{lemma}
The final tools that we frequently use throughout the paper are the following maximal inequalities for Gaussian and sub-Gaussian random variables. 
\begin{lemma}\label{lem:subg-max}
  Suppose that we are given random variables $(X_i)_{i \in S}$, with $\#S = n$, each of which satisfies
  \[\bb{P}(X_i > u) \le e^{\frac{-u^2}{2}}\]
  for all $u > 0$. Then we have 
  \begin{equation} \label{eq:subg-max} \bb{E}\left[\max_{i \in [n]} X_i\right] \le \sqrt{2\log n} + \sqrt{\pi/2}.
  \end{equation}
  In the special case where the $X_i$ are independent Gaussian random variables, we have 
  \begin{equation}\label{eq:gauss-max} \bb{E}\left[\max_{i \in [n]} X_i\right] \le \sqrt{2\log n}
  \end{equation}
\end{lemma}
\begin{proof}
  For \eqref{eq:gauss-max}, we refer the reader to \citet[Equation A.3]{chatterjee2014superconcentration} and its accompanying proof. For \eqref{eq:subg-max} we closely mimic the proof of \citet[Lemma 2.2.3]{Talagrand2014}.
  By a union bound and the fact that a probability can be at most $1$, for any $u > 0$ we have 
  \[\bb{P}\left(\max_{i \in S} X_i\mathbbm{1}\{X_i > 0\} > u\right) = \bb{P}\left(\max_{i \in S} X_i > u\right) \le ne^{\frac{-u^2}{2}} \wedge 1\]
  Then we can integrate the tail to obtain 
  \begin{align*}
    \bb{E}\left[\max_{i \in [n]} X_i\right] 
    &\le \bb{E}\left[\max_{i \in [n]} X_i\mathbbm{1}\{X_i > 0\} \right] \\
    &= \int_{0}^\infty \bb{P}\left(\max_{i \in S} X_i\mathbbm{1}\{X_i > 0\} > u\right)\,du \\
    &\le \int_{0}^\infty (ne^{\frac{-u^2}{2}} \wedge 1)\,du \\
    &\le A + \int_{A}^\infty ne^{-u^2/2}\,du \\
    &\le A + \frac{n}{A} \int_{A}^\infty ue^{-u^2/2}\,du
    \intertext{Making the substitution $v = u^2/2$, we obtain} 
    &= A + \frac{n}{A} \int_{A^2/2}^\infty e^{-v}\,du  \\
    &= A + \frac{n}{A} e^{-A^2/2}.
    \intertext{Choosing $A = \sqrt{2\log(n)}$ gives}
    &=\sqrt{2\log n} + \sqrt{1/(2\log n)}
  \end{align*}
  Since $1/(2\log n) \le \pi/2$ for $n \ge 2$, we deduce the result for all $n \ge 2$. Meanwhile, for the case $n=1$ we have that
  \[\int_0^\infty e^{-u^2/2} \le \sqrt{\pi/2}\] since the Gaussian density is symmetric and integrates to $1$. 
\end{proof}
\subsection{Other Useful Results}
\begin{lemma}\label{lem:quadratic-margin}
  Given $K$, a closed, convex subset of a Euclidean space $H$, and a point $y \in H$, let $\hat z$ denote the nearest point to $y$ in $H$. Then, for any $z \in K$,
  \[\norm{y-z}^2-\norm{y-\hat z}^2 \ge \norm{\hat z - z}^2.\]
\end{lemma}
\begin{proof}
  Note that by strict convexity of the squared norm, $\hat x$ must be unique. By translation, we may assume $y = 0$. There are two cases. If $0 \in K$ then $0 = \hat z$ and the statement is trivial. On the other hand, suppose $0 \not\in K$. Let $C$ be a minimal ball centered at $0$ which contains $\hat z$. Then, the tangent plane to $C$ at $\hat z$, which is perpendicular to $\hat z$ must separate $0$ from $K$ (or else, using convexity of $K$ and strict convexity of the squared norm, we can find a better point than $\hat z$ by interpolation). It follows that $\bk{\hat z}{z} > 0$. Finally, we may compute 
  \begin{align*} \norm{z}^2-\norm{\hat z}^2 
    &= \bk{z}{z}  + \bk{\hat z}{\hat z} \\
    &\ge \bk{z}{z} + \bk{\hat z}{\hat z} - 2\bk{\hat z}{z}
    &= \norm{z-\hat z}^2.
  \end{align*} 
  This completes the proof.
\end{proof}
\end{document}